\documentclass[11pt]{article}

\usepackage[preprint]{acl}

\usepackage{times}
\usepackage{latexsym}
\usepackage[T1]{fontenc}
\usepackage[utf8]{inputenc}
\usepackage{microtype}
\usepackage{inconsolata}
\usepackage{graphicx}

\usepackage{booktabs}
\usepackage{multirow}
\usepackage{xcolor}
\usepackage{colortbl}
\usepackage{tabularx}
\definecolor{oursbg}{HTML}{FFF7E3}
\definecolor{headbg}{HTML}{E3EAF2}
\definecolor{gainc}{HTML}{D9481C}
\definecolor{demph}{gray}{0.5}
\definecolor{dropc}{HTML}{1F8A8A}
\newcommand{\gain}[1]{\textsubscript{\textcolor{gainc}{$\uparrow$#1}}}
\newcommand{\gainsig}[2]{$^{#2}_{\textcolor{gainc}{\uparrow#1}}$}
\newcommand{\drop}[1]{\textcolor{dropc}{$\downarrow$#1}}
\newcommand{\oursrow}{\rowcolor{oursbg}}
\newcommand{\headrow}{\rowcolor{headbg}}
\usepackage{pifont}
\usepackage{amsmath}
\usepackage{amssymb}
\usepackage{amsthm}
\usepackage{bm}
\usepackage{enumitem}
\usepackage{xspace}

\newtheorem{proposition}{Proposition}

\newcommand{\ourmethod}{\textsc{Madeleine}\xspace}
\newcommand{\obs}[1]{\textbf{Obs.~\ding{\numexpr181+#1\relax}}}

\title{Madeleine: Learning Involuntary Recall for Conversational Memory\\from Simulated Lives}

\author{Zhiyun Shi \\ Nanyang Technological University, Singapore \\ \texttt{ZHIYUN004@e.ntu.edu.sg}}

\begin{document}
\maketitle

\begin{abstract}
A long-term conversational assistant must recall the right memory at the right moment, yet the memory that matters most is often not similar to what the user says now.
Current systems recover such associations by letting an LLM reason at write or read time, at a cost of hundreds to over a thousand LLM calls per memory bank and up to several thousand context tokens per query. %
We argue that association is a learnable relevance: the pointwise mutual information of memories under how human lives unfold.
We introduce \textbf{\ourmethod}, which learns \textbf{amortized association}: offline, an LLM \textbf{life simulator} writes simulated lives, whose cue--trigger pairs teach a query encoder a \textbf{residual association} on top of frozen similarity; online, it calls no LLM and plugs into any vector memory by replacing only the query encoder.
On LoCoMo-Plus under the official protocol, \ourmethod
\textbf{(I)} reaches 66.6 when plugged into HyperMem, the highest among all systems evaluated under this protocol; %
\textbf{(II)} used alone, reaches the score of HyperMem as released (52.4 vs.\ 52.9) with zero LLM calls and about 1/21 of its answer context; %
and \textbf{(III)} lifts T-Mem by 26.2 points and significantly outperforms the same untrained backbone inside both systems. %
\end{abstract}

\section{Introduction}\label{sec:intro}

In Proust's novel, the taste of a madeleine dipped in tea brings back the narrator's childhood in Combray, although the cake resembles nothing in the town~\citep{proust1913,e_proust2024}.
Psychology calls this \textit{involuntary} memory: a cue that is related to the present, but not similar to it, surfaces without effort~\citep{berntsen2009,collins1975}.
Conversational assistants now need the same ability.
A user once said ``learning to say no has really reduced my stress''; months later she writes ``I took on that project anyway and I'm exhausted''.
The two turns share no content, yet an assistant that remembers the first should answer the second differently (Figure~\ref{fig:intro}, left).
This matters in practice: assistants such as ChatGPT now keep long-term memories of their users~\citep{e_chatgptmemory}, and memory retrieval runs at every turn of every conversation, so recall must be both associative and cheap enough for each request.

\begin{figure*}[t]
\centering
\begin{minipage}[c]{0.565\textwidth}\includegraphics[width=\linewidth]{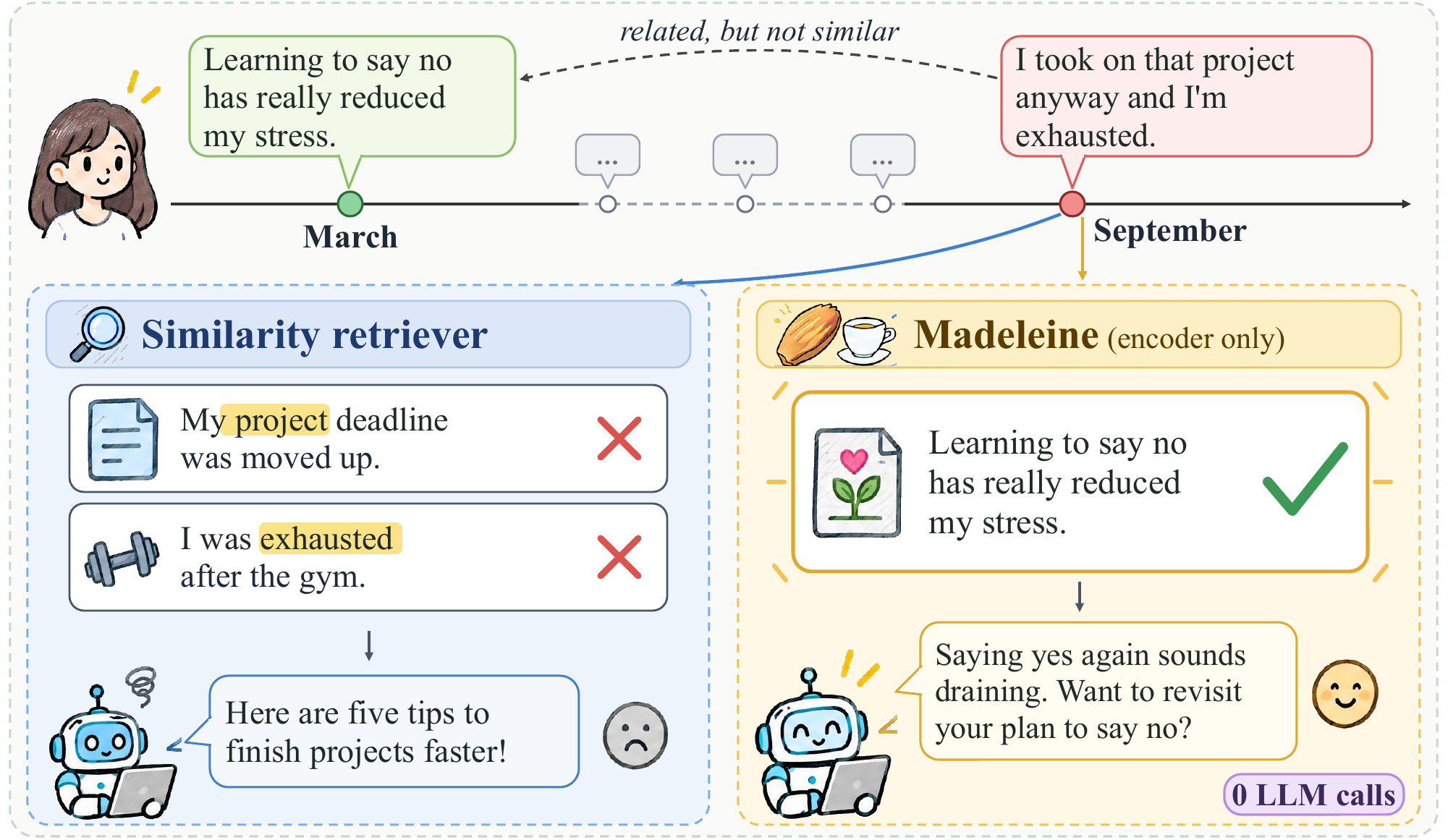}\end{minipage}\hfill
\begin{minipage}[c]{0.42\textwidth}\includegraphics[width=\linewidth]{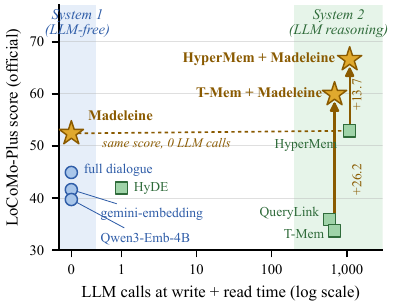}\end{minipage}
\vspace{-0.5em}
\caption{(\textbf{\textit{Left}}) Associative recall: the trigger shares no content with its cue. A similarity retriever returns memories that share surface words with the trigger, whereas \ourmethod recalls the cue without calling an LLM.
(\textbf{\textit{Right}}) LoCoMo-Plus score (official protocol) against LLM calls at write and read time (log scale). LLM-reasoning (System~2) memory systems lie to the right; standalone \ourmethod reaches the score of HyperMem as released with zero LLM calls, and plugging it into HyperMem and T-Mem (arrows) raises both scores at unchanged cost.}
\label{fig:intro}
\vspace{-0.8em}
\end{figure*}

Memory systems have approached this goal by adding ever more LLM reasoning: an LLM extracts salient facts~\citep{mem0,zep}, organizes memories into notes or hypergraphs~\citep{amem,hypermem}, predicts at write time the triggers under which a memory will be needed~\citep{tmem}, or rewrites each query into what the relevant memory might say~\citep{hyde,querylink}.
In the terms of dual-process theory~\citep{kahneman2011}, all of them realize association as \textit{System~2}, a deliberate LLM reasoning step spent on every memory or query.
It faces two dilemmas.

\textbf{Dilemma 1: similarity cannot see association.}
Underneath the LLM calls, every such system still retrieves with an embedder that equates relevance with semantic similarity.
On LoCoMo-Plus~\citep{locomoplus}, a benchmark built for associative recall, Qwen3-Embedding-4B~\citep{qwen3emb} and gemini-embedding~\citep{geminiemb} miss 33\% to 43\% of goal and value cues from their top 10, against about 11\% to 13\% of causal cues. %
On InMind~\citep{inmind}, the untrained 4B finds a personal fact in its top 10 for 96.8\% of direct questions but for only 5.6\% of indirect questions about the same fact. %
InMind observes that such associations rarely co-occur in text.

\textbf{Dilemma 2: System-2 remedies are costly and context-hungry.}
HyperMem calls an LLM 1,095 times to build one memory bank and feeds the reader 7,817 answer-input tokens per query. %
Its accuracy rests on this breadth rather than on ranking: with 541 tokens it scores 9.5 instead of 52.9 (Figure~\ref{fig:budget}). %
Read-time rewriting fares no better: QueryLink spends two LLM calls per query, yet its rewrites, encoded by Qwen3-Embedding-4B, score 34.4, below the 39.7 of the raw trigger. %

We take a different view. Human association is effortless (System~1), and it follows regularities of how lives unfold that are shared across people.
We therefore reformulate association as a \textbf{learnable relevance}.
The optimal InfoNCE critic is the pointwise mutual information (PMI) of its training distribution~\citep{infonce,poole2019}; an encoder trained on web text learns the PMI of web text, whereas one trained on pairs sampled from simulated lives learns the PMI of those lives, an approximation of \textbf{associative relevance} as the simulator portrays it (Section~\ref{sec:prelim}).
Since such pairs are not written down, we let an LLM simulate lives to write them.
The result is \textbf{amortized association}: an LLM thinks slowly once, offline, and a retriever recalls fast at every turn afterwards.
We instantiate this idea as \textbf{\ourmethod}: a \textbf{life simulator} writes (trigger, cue) pairs from several sources of simulated lives, and a \textbf{residual association} retriever adds a learned association score to the frozen similarity score, training only the query side so that the memory index is reused unchanged.

All results below use the official LoCoMo-Plus protocol with its gemini-2.5-flash judge (Figure~\ref{fig:intro}, right).
\ourmethod is
\textbf{\ding{182} state-of-the-art on LoCoMo-Plus}: plugged into HyperMem, it raises HyperMem from 52.9 to 66.6, the highest among all systems we evaluate under the official protocol; %
\textbf{\ding{183} LLM-free}: standalone, it reaches the score of HyperMem as released (52.4 vs.\ 52.9, not significantly different for any seed) with zero LLM calls at write and read time and about 1/21 of its answer context (375 vs.\ 7,817 tokens); %
\textbf{\ding{184} plug-and-play}: it lifts T-Mem by 26.2 points, and inside both systems it is significantly above the same backbone without training ($p=0.012$, and $0.047$ for T-Mem with three scenes); %
\textbf{\ding{185} harmless}: on ordinary LoCoMo QA, its three 4B seeds score 58.9 to 60.0 against 59.1 for the untrained backbone. %
Our contributions are:
\begin{itemize}[leftmargin=*,itemsep=0pt,topsep=2pt]
\item \textbf{\textit{Paradigm Reformulation}:} We recast associative memory retrieval as a learnable relevance, the PMI of memories under life trajectories, and propose \textbf{amortized association}, which moves the LLM's reasoning offline into a query encoder.
\item \textbf{\textit{Practical Solution}:} We propose \ourmethod, which learns this relevance from life-simulated supervision as a residual on top of frozen similarity and plugs into existing memory systems by replacing only the query encoder.
\item \textbf{\textit{Experimental Evaluation}:} Under the official LoCoMo-Plus protocol, \ourmethod achieves the highest score when plugged into HyperMem and reaches the score of HyperMem as released with zero LLM calls; controls attribute the gain to learned association (Section~\ref{sec:rq3}).
\end{itemize}
\section{Related Work}
\label{sec:related}

\paragraph{LLM-Constructed Memory Systems.}
Long-term conversational memory is mostly built by letting an LLM restructure the dialogue as it arrives:
\textbf{(I) fact and OS-style memories}: Mem0~\citep{mem0} and Zep~\citep{zep} extract salient facts or temporal graphs, and MemGPT~\citep{memgpt} and MemOS~\citep{memos} page memory like an operating system;
\textbf{(II) structured memories} such as A-Mem~\citep{amem}, SeCom~\citep{secom}, HippoRAG~\citep{hipporag}, AssoMem~\citep{assomem}, HyperMem~\citep{hypermem}, RippleMem~\citep{d_ripplemem} and CABLE~\citep{cable} segment or link memories into notes, knowledge or clue graphs, hypergraphs, recollection graphs or antecedent links;
\textbf{(III) anticipatory memories} such as T-Mem~\citep{tmem} and RootMem~\citep{d_rootmem} ask an LLM at write time to predict the triggers under which a memory will be needed, as Doc2Query~\citep{doc2query} predicts the queries a document answers.
The concurrent JustMem~\citep{justmem} chooses an access mode per query, and EdgeMem~\citep{edgemem} builds an LLM-free hypergraph that still retrieves by content, time and episode.
The LLM-built systems pay for association with write-time LLM calls, and a cue the LLM did not anticipate stays reachable only by similarity; \ourmethod replaces only their query encoder and is \textbf{plugged into} them.

\paragraph{Read-Time Reasoning and Query Rewriting.}
HyDE~\citep{hyde} and Query2doc~\citep{query2doc} let an LLM write a hypothetical document as the query; for conversational memory, LLMs generate keywords and implied events~\citep{querylink}, prospective probes~\citep{pgr} or abductive searches~\citep{adar}.
Information retrieval has seen this trajectory before: vocabulary mismatch was first attacked by query expansion with relevance and pseudo-relevance feedback~\citep{d_rocchio1971,d_lavrenko2001}, which were largely superseded by dense retrievers such as DPR~\citep{dpr} that learn the mismatch into the encoder.
Memory systems are now in their \textit{query-expansion era}, re-deriving every association with an LLM; \ourmethod takes the DPR step for \textbf{associative relevance}.

\paragraph{Retrievers Trained for Reasoning or Memory.}
Reasoning-intensive retrievers~\citep{reasonir,diver,reasonembed} are trained on synthesized queries whose relevance requires reasoning, and memory-specific models include the embedder HiNS~\citep{hins}, the reranker MemReranker~\citep{memreranker} and the feedback-driven RMM~\citep{rmm}, EAR~\citep{ear} and EARM~\citep{d_earm}; distilling a reader's judgments into the retriever~\citep{fidkd} is the closest precedent for amortizing an LLM's reasoning.
They learn relevance that is visible in the text, or rerank candidates that similarity has already retrieved; none is trained on cues and triggers that share no topic, and each one we evaluate falls below \ourmethod (Section~\ref{sec:rq3}). Query-side tuning over a frozen index follows DensePhrases~\citep{densephrases} and frozen-encoder adapters~\citep{searchadaptor}.

\paragraph{Synthetic Supervision and Simulated Lives.}
LLM-generated data is a standard way to train retrievers~\citep{inpars,promptagator,e5mistral,gecko}, but its queries are written from, and hence similar to, their passages.
LLM-simulated people~\citep{d_generativeagents,locomo,personamem,d_lifesim} have been used to simulate behaviour or build benchmarks, not to train retrievers.
Distilling System~2 into System~1~\citep{d_system2distill} is closest in spirit to amortized association, but it distills into a generator rather than an encoder.
Associative recall is evaluated by LoCoMo-Plus~\citep{locomoplus}, InMind~\citep{inmind} and LoCoMo-Conv~\citep{locomoconv}; where InMind finds such associations rarely written down, \ourmethod writes them down with \textbf{life-simulated supervision}.
\section{Problem Formulation and Preliminaries}
\label{sec:prelim}

\paragraph{Associative memory retrieval.}
A long-term conversational assistant keeps, for each user, a memory bank $\mathcal{M}=\{m_1,\dots,m_N\}$ whose entries are raw dialogue turns stored verbatim as ``speaker: utterance''.
At every turn the user produces a new utterance $q$, which we call the \textbf{trigger}.
Somewhere in $\mathcal{M}$ lies a \textbf{cue} $m^\star$, an earlier utterance that the assistant must recall to respond appropriately to $q$~\citep{locomoplus}.
A retriever scores every memory with $s(q,m)$ and passes the top-$K$, together with $q$, to a reader LLM. Our goal is to maximize $\Pr\!\left[m^\star \in \operatorname{top\text{-}}K_{m\in\mathcal{M}}\, s(q,m)\right]$ under a \textbf{System-1 constraint}: $s$ must be computed by encoders alone, so that neither writing nor reading invokes an LLM.

\paragraph{The similarity assumption.}
Vector memory systems instantiate $s$ with a pretrained text embedder $E$~\citep{mem0,hypermem,tmem}:
\begin{equation}
\label{eq:sim}
s_0(q,m)=\cos\!\big(E(q),E(m)\big).
\end{equation}
This choice equates relevance with semantic similarity; a cue carries \textbf{associative relevance} when it is relevant to the trigger but $s_0$ ranks many unrelated memories above it, as in Section~\ref{sec:intro}.

\paragraph{Contrastive learning estimates pointwise mutual information.}
Let $p(q,m)$ be a joint distribution over (trigger, cue) pairs.
The InfoNCE objective~\citep{infonce} scores the positive $m^+\!\sim p(m\mid q)$ against $n-1$ negatives drawn from the marginal $p(m)$:
\begin{equation}
\label{eq:infonce}
\mathcal{L}(f)=-\,\mathbb{E}\left[\log\frac{\exp\!\big(f(q,m^+)/\tau\big)}{\sum_{j=1}^{n}\exp\!\big(f(q,m_j)/\tau\big)}\right],
\end{equation}
where $\tau$ is a temperature.
Its optimal critic satisfies~\citep{infonce,poole2019}
\begin{equation}
\label{eq:pmi}
f^\star(q,m)/\tau=\operatorname{PMI}_p(q;m)+c(q),
\end{equation}
where $\operatorname{PMI}_p(q;m)=\log p(m\mid q)/p(m)$ and $c(q)$ does not affect the ranking of memories for a fixed trigger.
Hence an embedder trained on generic query--passage pairs~\citep{e5mistral,qwen3emb} approximates the PMI of \emph{generic text} (``similar, therefore relevant''), while the same objective on pairs sampled from how lives unfold approximates the PMI of \emph{life trajectories}, which is associative relevance; in practice we sample simulated lives, so the target is the simulator's PMI.
The question is how to learn the second without discarding the first.

\paragraph{Residual view.}
We keep $s_0$ fixed and learn only an additive association term $a$.

\begin{proposition}[Residual association]
\label{prop:residual}
Fix $s_0$ and let $p_{\mathrm{life}}$ be the life-trajectory distribution. Among scores of the form $s(q,m)=s_0(q,m)+a(q,m)$ with an unrestricted $a$, the InfoNCE minimizer under $p_{\mathrm{life}}$, with negatives drawn from its marginal, is
\begin{equation}
\label{eq:residual}
a^\star(q,m)=\tau\cdot\operatorname{PMI}_{p_{\mathrm{life}}}(q;m)-s_0(q,m)+\tau\cdot c(q).
\end{equation}
\end{proposition}

\noindent\textit{Proof sketch.} Apply Eq.~\eqref{eq:pmi} to the summed score $s$ and subtract the fixed $s_0$. \hfill$\square$

\smallskip
\noindent The association term thus only encodes the part of life-trajectory PMI that similarity leaves unexplained.
Hence training must use the summed test-time score: an $a$ trained alone targets $\tau\cdot\operatorname{PMI}$, so adding $s_0$ at test time counts similarity twice.
\section{Method}
\label{sec:method}

\subsection{Overview}
\label{sec:method-overview}

\ourmethod rests on \textbf{amortized association}: the associative reasoning that current memory systems ask an LLM to perform at every write or read (System~2) is performed once, offline, and distilled into a query encoder that recalls without any LLM call (System~1).
As Figure~\ref{fig:framework} shows, it has three components:
\ding{182} the life simulator, which produces \textbf{life-simulated supervision} (Section~\ref{sec:simulator});
\ding{183} the \textbf{residual association} retriever, which trains only the query side (Section~\ref{sec:retriever});
and \ding{184} a contrastive objective with \textbf{same-source negatives} (Section~\ref{sec:objective}).

\begin{figure*}[t]
\centering
\includegraphics[width=\textwidth]{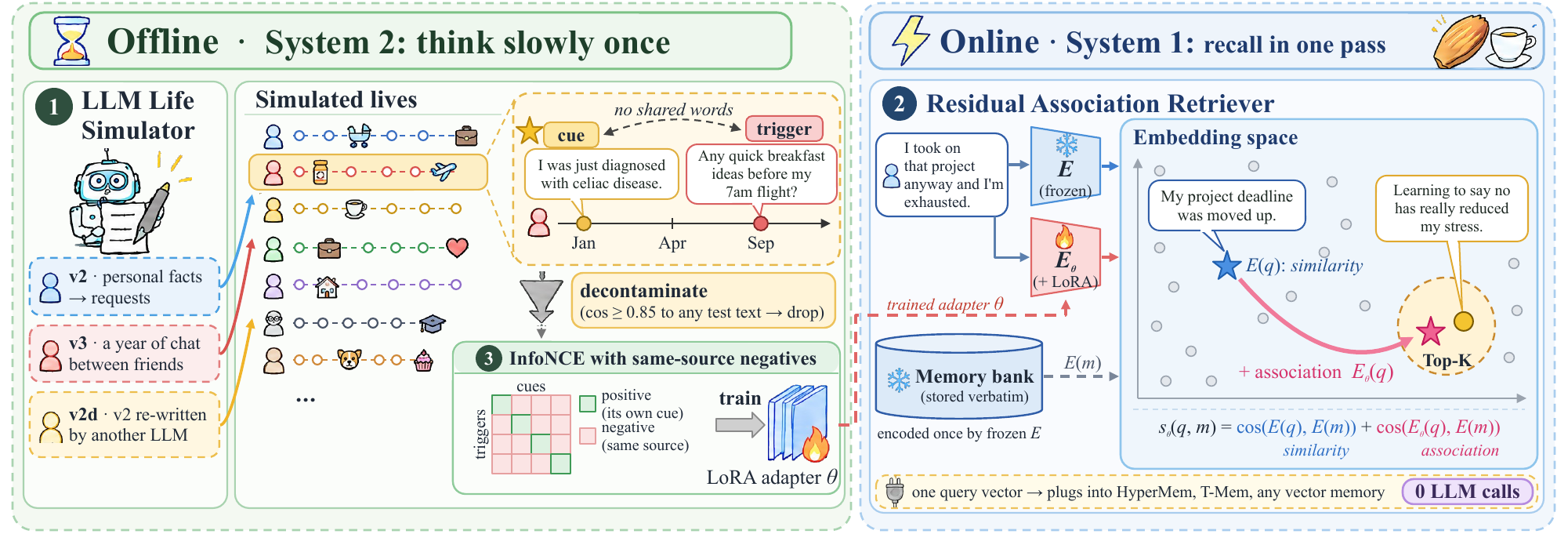}
\caption{Overview of \ourmethod.
(\textbf{\textit{Left, offline / System~2}}) An LLM life simulator writes simulated lives from three sources; each earlier statement becomes a cue and each later utterance a trigger that does not hint at it. The decontaminated pairs train a query-side LoRA adapter with InfoNCE over same-source facts.
(\textbf{\textit{Right, online / System~1}}) Memories are stored verbatim and embedded once by the frozen encoder $E$; the trigger is encoded by $E$ and by the adapted $E_\theta$, and the similarity and association scores are summed. A single query vector is the plug-in interface to a host system such as HyperMem.}
\label{fig:framework}
\end{figure*}

\subsection{Life Simulator}
\label{sec:simulator}

Training pairs should be sampled from how lives unfold (Section~\ref{sec:prelim}), so we let an LLM write them.
Each call simulates one person and returns about ten earlier statements, each of a kind sampled from about twenty aspects of a life (from a diagnosis or a debt to a fear or a principle) and each with three later utterances that the statement should bring to mind and one direct question about it. %
We build three sources that differ in what they simulate and in which LLM writes them:
\textbf{v2} (concrete personal facts such as an illness, a diet or an obligation, followed by later requests to an assistant; written by gpt-4.1-mini~\citep{a_gpt41}),
\textbf{v3} (a year of chat between two friends, in which a goal, plan, experience or value develops into progress, a setback, a change of mind or a dilemma; written by DeepSeek-V4.1-Flash~\citep{deepseekv4}),
and \textbf{v2d} (the v2 prompt and personas re-written by DeepSeek-V4.1-Flash). %
All prompts forbid the later utterance from mentioning or hinting at the earlier topic, and none copies the wording or examples of any benchmark; direct questions keep ordinary lookup in the training distribution.
We prefer generators whose earlier and later utterances have a low mean cosine on a small pilot, since a high cosine teaches similarity rather than association, and we remove every training row whose cosine with any test text is at least 0.85. %
Each source has about 32K rows and costs under \$2 to write (Appendix~\ref{app:simulator}), and we use all three. %

\subsection{Residual Association Retriever}
\label{sec:retriever}

Let $E$ be a frozen pretrained embedder and $E_\theta$ the same network with a LoRA adapter $\theta$.
\ourmethod scores a memory $m$ for a trigger $q$ by
\begin{equation}
\label{eq:score}
\resizebox{0.93\columnwidth}{!}{$s_\theta(q,m)=\underbrace{\cos\!\big(E(q),E(m)\big)}_{\text{similarity}}+\underbrace{\cos\!\big(E_\theta(q),E(m)\big)}_{\text{association}}$},
\end{equation}
the residual form $s_0+a$ of Proposition~\ref{prop:residual}.
Memories are always encoded by the frozen $E$; only the trigger passes through the adapter.
We use Qwen3-Embedding-4B~\citep{qwen3emb} as $E$ and prefix every trigger with one fixed task instruction, as its convention prescribes; the untrained baseline uses the same instruction.

\paragraph{Why only the query side.}
Training only the query encoder yields \textbf{(I) index reuse}, since memory vectors are those of the base embedder, and \textbf{(II) zero write cost}, since writing stays one forward pass of $E$.
It also learns more association: adapting the memory side as well lowers associative recall on 4B (Section~\ref{sec:rq4}), and training with the summed test-time score makes the adapter learn the residual rather than a second copy of similarity (Proposition~\ref{prop:residual}).

\subsection{Training Objective}
\label{sec:objective}

Each training row is a pair $(q_i,m_i^+)$ from a source $\sigma(i)\in\{\text{v2},\text{v3},\text{v2d}\}$, where $q_i$ is a trigger or a direct question and $m_i^+$ its fact. With $\mathcal{F}_\sigma$ the training facts of source $\sigma$, we minimize InfoNCE (Eq.~\eqref{eq:infonce}) with the summed score over the negatives
\begin{equation}
\label{eq:negatives}
\resizebox{0.9\columnwidth}{!}{$\mathcal{N}_i=\{m\in\mathcal{F}_{\sigma(i)}\!\setminus\!\{m_i^+\}:\cos(E(m),E(m_i^+))\le 0.9\}$}. %
\end{equation}
Because the memory side is frozen, fact embeddings are computed once, so every row competes against the whole fact bank of its source (7{,}522 facts for v2) rather than a few dozen in-batch facts. %
Restricting negatives to the row's \textbf{own source}, which applies Proposition~\ref{prop:residual} within each source, matters when sources are mixed, because a fact from another source may describe a related aspect of life, and treating it as a negative would penalize a correct association.
We train LoRA of rank 32 with $\tau=0.05$, batch size 128 and learning rate $10^{-4}$ for one epoch (Appendix~\ref{app:hparams}). %

\subsection{Deployment and Cost}
\label{sec:deploy}

Since all vectors in Eq.~\eqref{eq:score} have unit norm, the summed score equals an inner product with a single query vector, $s_\theta(q,m)=\langle E(q)+E_\theta(q),E(m)\rangle$.
A host memory system therefore adopts \ourmethod as a \textbf{plug-in}: it replaces its query vector with the normalized $E(q)+E_\theta(q)$ and keeps its memories, index, lexical retrieval and any graph built on top; if it already embeds memories with the same backbone, as HyperMem does, its index is reused unchanged.
Writing costs one forward pass of $E$ and reading one extra query encoding, with zero LLM calls; the simulator, our only LLM, runs once offline for all users.
\begin{table*}[t]
\centering
\caption{Main results on LoCoMo-Plus (401 cognitive questions, official protocol: gpt-4.1-mini reader, gemini-2.5-flash judge); \textbf{Causal}--\textbf{Value}: the four relation types, \textbf{All}: overall. \textbf{Write}: LLM calls per memory bank; \textbf{Read}: LLM calls per query, excluding the answer call; \textbf{Tokens}: answer-input tokens per query. Retrievers pass $K{=}3$ turns to the reader. Standalone \ourmethod{} is the mean of three seeds; plug-in rows replace the host system's retrieval encoder with \ourmethod{} (seed~0), re-encoding T-Mem's memories with the frozen backbone. \textcolor{gainc}{$\uparrow$}: gain over the same system without \ourmethod{} (standalone: over the untrained backbone). Best per column in \textbf{bold}, second \underline{underlined}. T-Mem and JustMem report 74.81 and 62.59 under their own, non-comparable protocols (Appendices~\ref{app:protocol}, \ref{app:tmem}).}
\label{tab:main}
\footnotesize
\setlength{\tabcolsep}{3.8pt}
\setlength{\aboverulesep}{0pt}\setlength{\belowrulesep}{0pt}
\renewcommand{\arraystretch}{1.0}
\newlength{\gw}\setlength{\gw}{2.1cm}
\begin{tabular}{p{\gw}lrrrccccc}
\toprule
\headrow \textbf{Group} & \textbf{Method} & \textbf{Write} & \textbf{Read} & \textbf{Tokens} & \textbf{Causal} & \textbf{State} & \textbf{Goal} & \textbf{Value} & \textbf{All} \\
\midrule
 & gemini-embedding & 0 & 0 & 387 & 55.4 & 52.0 & 26.0 & 33.0 & 41.6 \\ %
 & text-embedding-3-large & 0 & 0 & 396 & 58.4 & 45.0 & 30.0 & 31.0 & 41.1 \\ %
 & Qwen3-Embedding-4B & 0 & 0 & 378 & 54.5 & 37.0 & 29.0 & 38.0 & 39.7 \\ %
 & Qwen3-Embedding-8B & 0 & 0 & 386 & 50.5 & 36.0 & 28.0 & 35.0 & 37.4 \\ %
\multirow{-5}{\gw}{\raggedright\itshape (a) Embedding models} & bge-m3 & 0 & 0 & 385 & 31.7 & 28.0 & 9.0 & 19.0 & 21.9 \\ %
\midrule
\multirow{-1}{\gw}{\raggedright\itshape (b) Full context} & Full dialogue & 0 & 0 & 18,776 & 49.5 & 49.0 & 41.0 & 40.0 & 44.9 \\ %
\midrule
 & QueryLink, native union & 590 & 2 & 1,301 & 39.6 & 41.0 & 32.0 & 31.0 & 35.9 \\ %
 & QueryLink, Qwen3-4B & 590 & 2 & 339 & 44.6 & 35.0 & 26.0 & 32.0 & 34.4 \\ %
\multirow{-3}{\gw}{\raggedright\itshape (c) Read-time LLM rewriting} & HyDE, Qwen3-4B & 0 & 1 & 370 & 51.5 & 45.0 & 42.0 & 29.0 & 41.9 \\ %
\midrule
 & ReasonIR-8B & 0 & 0 & 399 & 54.5 & 48.0 & 29.0 & 31.0 & 40.6 \\ %
\multirow{-2}{\gw}{\raggedright\itshape (d) Trained retrievers} & EAR-style adapter (fused) & 0 & 0 & 377 & 56.4 & 43.0 & 37.0 & 44.0 & 45.1 \\ %
\midrule
\oursrow \multirow{-1}{\gw}{\raggedright\itshape Ours} & \ourmethod{} (standalone) & 0 & 0 & 375 & 59.7\gain{5.2} & 57.3\gain{20.3} & 45.7\gain{16.7} & 46.7\gain{8.7} & 52.4\gain{12.7} \\ %
\midrule
 & T-Mem, 3 scenes & 689 & 0 & 307 & 48.5 & 40.0 & 25.0 & 21.0 & 33.7 \\ %
\oursrow  & \quad + \ourmethod{} & 689 & 0 & 305 & 62.4\gain{13.9} & \underline{68.0}\gain{28.0} & \underline{52.0}\gain{27.0} & \underline{57.0}\gain{36.0} & \underline{59.9}\gain{26.2} \\ %
 & HyperMem (released) & 1,095 & 0 & 7,817 & \underline{66.3} & 53.0 & 46.0 & 46.0 & 52.9 \\ %
\oursrow \multirow{-4}{\gw}{\raggedright\itshape (e) LLM-built memory systems} & \quad + \ourmethod{} & 1,095 & 0 & 7,814 & \textbf{72.3}\gain{6.0} & \textbf{71.0}\gain{18.0} & \textbf{59.0}\gain{13.0} & \textbf{64.0}\gain{18.0} & \textbf{66.6}\gain{13.7} \\ %
\bottomrule
\end{tabular}
\end{table*}

\section{Experiments}\label{sec:exp}

We study effectiveness (\textbf{RQ1}), cost (\textbf{RQ2}), whether the gain is \emph{learned association} (\textbf{RQ3}), design choices (\textbf{RQ4}), robustness and harmlessness (\textbf{RQ5}), and where association helps (\textbf{RQ6}).

\subsection{Experimental Setup}\label{sec:exp-setup}

\paragraph{Benchmark and protocol.}
We evaluate on LoCoMo-Plus~\citep{locomoplus}, the benchmark that targets associative recall in long conversations. It contains 401 cognitive questions. %
Each question inserts a \emph{cue} dialogue into one of the ten LoCoMo conversations~\citep{locomo} and later presents a \emph{trigger} that is semantically distant from the cue. The questions cover four relation types: causal (101 questions), state, goal and value (100 each). %
We run the official code verbatim: the reader is gpt-4.1-mini~\citep{a_gpt41} with temperature 0.3 and max\_tokens 1024, and the official judge prompt is scored by gemini-2.5-flash~\citep{b_gemini25}. %
The score is the percentage of questions judged correct, and every baseline is re-run under this same code, so systems differ only in what they place in the context slot reserved for the dialogue.
Memory units are raw turns, and the trigger is never stored (Appendix~\ref{app:protocol}).

\paragraph{Baselines.}
We compare against five groups of baselines.
\textbf{(a) Embedding models}: gemini-embedding~\citep{geminiemb}, text-embedding-3-large~\citep{openaiemb}, Qwen3-Embedding-4B and -8B~\citep{qwen3emb}, and bge-m3~\citep{bgem3}.
\textbf{(b) Full context}: the whole dialogue is given to the reader, as in the official baseline.
\textbf{(c) Read-time LLM rewriting}: QueryLink~\citep{querylink} with its original all-MiniLM-L6-v2 encoder~\citep{b_sbert}, its native multi-view union, or a frozen Qwen3-Embedding-4B, and HyDE~\citep{hyde} on our backbone.
\textbf{(d) Trained retrievers and rerankers}: ReasonIR~\citep{reasonir}, DIVER~\citep{diver}, ReasonEmbed~\citep{reasonembed}, an EAR-style adapter~\citep{ear} trained on exactly our data and objective, and Qwen3-Reranker~\citep{qwen3emb} and MemReranker~\citep{memreranker}.
\textbf{(e) LLM-built memory systems}: HyperMem~\citep{hypermem} and T-Mem~\citep{tmem}, run from their official code and prompts. HyperMem uses its default retrieval budget, and T-Mem passes its top-3 scenes to the reader to match $K$.

\paragraph{Implementation details.}
\ourmethod{} uses Qwen3-Embedding-4B~\citep{qwen3emb} as the frozen memory encoder and adds a query-side LoRA~\citep{lora} of rank 32. %
We train three seeds (Section~\ref{sec:objective}; Appendix~\ref{app:hparams}), pass $K{=}3$ retrieved turns to the reader in the main setting, and test significance with the exact McNemar test~\citep{mcnemar1947} per question.

\subsection{RQ1: Effectiveness}\label{sec:rq1}

Table~\ref{tab:main} reports the main systems under the official protocol; Table~\ref{tab:routes} lists every other route.

\obs{1} \textbf{Plugged into HyperMem, \ourmethod{} achieves the highest score under the official protocol.}
Replacing only HyperMem's query encoder with ours raises its score from 52.9 to 66.6 ($p=1.7{\times}10^{-7}$; three-seed mean 66.2), the best among all systems we evaluate and in every column of Table~\ref{tab:main}. %
The same swap lifts T-Mem from 33.7 to 59.9 ($p=7{\times}10^{-17}$). %

\obs{2} \textbf{Standalone, \ourmethod{} reaches the score of HyperMem as released and, at $K{=}3$, significantly outperforms every LLM-free baseline.}
With three retrieved turns, \ourmethod{} scores 52.4 on average over three seeds and the released HyperMem 52.9; the difference is not significant for any seed ($p=0.93$, $0.74$ and $1.0$; seed~0: $-0.5$ points, 95\% CI $[-6.2,+5.2]$). %
Every seed is significantly above its untrained backbone (by 12.0 to 13.5 points), the strongest embedding model (by 10.0 to 11.5) and the full dialogue (by 6.7 to 8.2; $p\le0.027$ in all nine comparisons), and above every method in (c) and (d) (Section~\ref{sec:rq3}). %

\subsection{RQ2: Efficiency}\label{sec:rq2}

\begin{figure}[t]
\centering
\includegraphics[width=\linewidth]{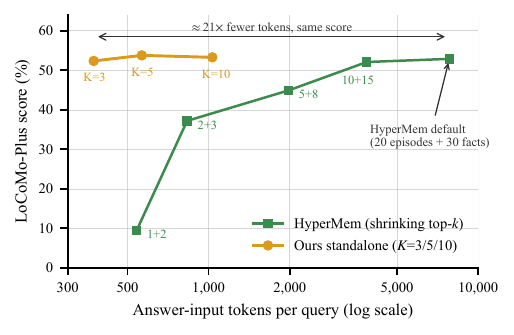}
\caption{Score versus answer-input tokens per query on LoCoMo-Plus (official protocol). The HyperMem curve shrinks its episode and fact budgets in its default 20:30 ratio; \ourmethod{} is standalone (three-seed means, $K{=}3/5/10$).}
\label{fig:budget}
\end{figure}

\obs{3} \textbf{\ourmethod{} reaches the score of HyperMem as released with about $1/21$ of its context and zero LLM calls.}
Standalone \ourmethod{} gives the reader 375 answer-input tokens per query, whereas the released HyperMem gives 7,817 for a score that is not significantly different, about 21 times more. %
HyperMem cannot close this gap by passing less evidence (Figure~\ref{fig:budget}): shrinking its own retrieval budget, it needs 3,862 tokens to tie \ourmethod{} ($p=1.0$), and below 2,000 tokens it drops sharply, to 44.9 at 1,977 tokens, which is significantly below \ourmethod{} ($p=0.022$), and to 9.5 at 541 tokens. %
HyperMem relies on breadth, not ranking: its default budget brings the cue episode into context for 73.6\% of the questions, \ourmethod{}'s three turns for 74.8\%. %

At write time, HyperMem, T-Mem and QueryLink call an LLM 1,095, 689 and 590 times per memory bank, whereas \ourmethod{} calls none (Table~\ref{tab:main}). %
At read time \ourmethod{} adds one encoder pass (70.5\,ms against 37.9\,ms for the untrained embedder), whereas one QueryLink rewriting call takes 1.27\,s (median). %

\subsection{RQ3: Is the Gain Learned Association?}\label{sec:rq3}

\begin{table}[t]
\centering
\caption{Same-backbone control on LoCoMo-Plus (official protocol). Within each row only the retrieval encoder changes: the system's \textbf{original} one (bge-m3 for T-Mem, Qwen3-Embedding-4B without its query instruction for HyperMem), the \textbf{untrained} Qwen3-Embedding-4B with the instruction, or \ourmethod{} (seed~0) on that backbone; memories are encoded by the frozen encoder of each column. \textcolor{gainc}{$\uparrow$}: gain over the untrained encoder; $^{*}p<0.05$, $^{**}p<10^{-4}$ (exact McNemar). All gains over the original encoder: $p<10^{-6}$.}
\label{tab:plugin}
\small
\setlength{\tabcolsep}{4pt}
\setlength{\aboverulesep}{0pt}\setlength{\belowrulesep}{0pt}
\renewcommand{\arraystretch}{1.15}
\begin{tabular}{lcc>{\columncolor{oursbg}}c}
\toprule
\headrow \textbf{System} & \textbf{Original} & \textbf{Untrained} & \textbf{\ourmethod{}} \\
\midrule
Standalone      & --   & 39.7 & 52.4\gainsig{12.7}{**} \\ %
T-Mem, 3 scenes & 33.7 & 53.9 & 59.9\gainsig{6.0}{*} \\ %
HyperMem        & 52.9 & 60.6 & \textbf{66.6}\gainsig{6.0}{*} \\ %
\bottomrule
\end{tabular}
\end{table}

Two controls rule out a stronger backbone, a better query format or extra computation as the cause.

\obs{4} \textbf{Within every system, \ourmethod{} significantly outperforms the same backbone without learned association.}
Table~\ref{tab:plugin} swaps only the retrieval encoder of each system; T-Mem's memories are then re-encoded by the frozen 4B.
Switching to the untrained Qwen3-Embedding-4B already helps both systems, lifting T-Mem from 33.7 to 53.9 because its original bge-m3 encoder is weaker, and HyperMem from 52.9 to 60.6 because its original code omits the query instruction that Qwen3-Embedding expects. %
\ourmethod{} adds a further 6.0 points over this same backbone in T-Mem ($p=0.047$) and HyperMem ($p=0.012$). %

\obs{5} \textbf{Every alternative route falls significantly below \ourmethod{}, and rerankers even hurt it.}
Table~\ref{tab:routes} (Appendix~\ref{app:routes}) compares four such families at three retrieved turns.
\textit{(i) Read-time LLM reasoning} does not substitute for learned association: the best QueryLink configuration scores 35.9 although it makes two LLM calls per query, and HyDE on our backbone scores 41.9, which is not significantly above the untrained backbone ($p=0.46$). %
\textit{(ii) Reasoning-oriented retrievers} such as ReasonIR do not significantly exceed the untrained backbone either. %
\textit{(iii) Rerankers} lower \ourmethod{} from 52.4 to 40.9 with Qwen3-Reranker-4B and to 31.9 with MemReranker-4B, because a cross-encoder judges whether a memory matches the trigger, while a cue by design does not. %
\textit{(iv) The same data in a different structure} is not enough: an EAR-style adapter trained on our exact data and objective ties the untrained backbone and, even with our fused score, stays 7.2 points below \ourmethod{} ($p=0.019$). %
The two controls show that the gain comes from association learned as a residual on the query side, not from the backbone, read-time LLM calls, or the training data alone.

\subsection{RQ4: Which Design Choices Matter?}
\label{sec:rq4}

Table~\ref{tab:ablation} (Appendix~\ref{app:ablation}) changes one design choice at a time and reports judge-independent Recall@10 on 0.6B, plus 4B for the decisive choices.

\obs{6} \textbf{Association must be learned on the query side only, as a residual on top of similarity.}
Training the memory side as well learns little association: on 4B a two-sided LoRA lowers LoCoMo-Plus cue Recall@10 from 71.1 to 69.6, whereas freezing the memory side lifts it to 78.8 with the same data, and on 0.6B full fine-tuning of both sides also lowers ordinary LoCoMo lookup from 67.3 to 55.1. %
Training the association term alone and summing only at test time costs 5.3 points (54.1 versus 59.4), as Proposition~\ref{prop:residual} predicts. %

\obs{7} \textbf{Mixing simulator sources helps only with same-source negatives.}
Naively mixing v2 and v3 is worse than either alone (50.1 versus 59.4 and 57.6), because related facts of the other source become false negatives; with same-source negatives the same mixture reaches 63.6, and a third source protects ordinary QA on 4B (Appendix~\ref{app:ablation}). %

\subsection{RQ5: Is the Gain Robust, and Is It Harmless?}
\label{sec:rq5}

\obs{8} \textbf{The gain holds across seeds, readers and backbones, and is largest when the reader sees few memories.}
All three seeds are significant at $K=3$ (Section~\ref{sec:rq1}), and with GPT-4o~\citep{gpt4o} as reader \ourmethod\ scores 50.1 against 40.4 for the untrained 4B ($p=9.1\times10^{-4}$). %
With more memories the untrained backbone catches up (seed means 53.8 versus 47.4 at $K=5$ and 53.3 versus 49.9 at $K=10$), because \ourmethod's advantage lies at the top of the ranking, where it places the cue first for 61.2\% of the items against 41.1\% (three-seed mean). %
Across backbones, cue Recall@10 on 0.6B, 4B and 8B reaches 60.6, 83.1 and 84.0, and 8B rises end to end from 36.7 (local run) to 53.4 ($p=2.4\times10^{-8}$). %

\obs{9} \textbf{On the 4B backbone, learned association does not hurt ordinary QA.}
On the original LoCoMo questions ($K=10$), the untrained 4B scores 59.1 and the three seeds of \ourmethod\ score 58.9, 59.2 and 60.0 ($p=0.86$, $0.96$ and $0.29$); 8B results are in Appendix~\ref{app:full}. %

\subsection{RQ6: Where Does Association Help?}
\label{sec:rq6}

\begin{figure}[t]
\centering
\includegraphics[width=\columnwidth]{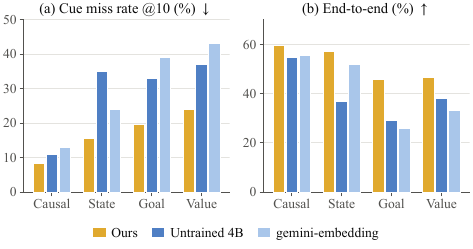}
\vspace{-1.5em}
\caption{LoCoMo-Plus by relation type ($\approx$100 items each): (a) cues missing from the top 10; (b) end-to-end score ($K=3$, three-seed mean for \ourmethod).}
\label{fig:relation}
\vspace{-1em}
\end{figure}

\obs{10} \textbf{Similarity fails on goals and values, and that is where \ourmethod\ helps.}
The untrained 4B and gemini-embedding miss only 10.9\% and 12.9\% of causal cues from the top 10 (Figure~\ref{fig:relation}). %
Goal and value cues are the blind spot, with miss rates of 33\% to 43\%, which \ourmethod\ lowers to 19.7\% and 24.0\%; end to end, on goal items it scores 45.7 against 29.0 and 26.0, whereas on causal items all three lie between 54.5 and 59.7. %
The gain is not tied to one simulator prompt: v2 alone, built around concrete conditions and requests to an assistant, lifts cue recall on all four relation types on 0.6B and most on state cues on 4B (Appendix~\ref{app:ablation}). %
In item \#272, John takes salsa lessons for his cousin's wedding and a month later gets everyone dancing at the office party; the untrained 4B ranks this cue 39th, whereas all three seeds of \ourmethod\ rank it first. %

On LoCoMo-Conv, whose implicit queries still name the topic, \ourmethod\ ties text-embedding-3-large, so the gain is specific to associative relevance (Appendix~\ref{app:other}). %
\section{Conclusion}\label{sec:conclusion}

We have argued that the associations a conversational memory needs are a learnable relevance, not something an LLM must re-derive at every turn.
\ourmethod realizes this as amortized association: a life simulator lets an LLM think slowly once, offline, and a query-side residual association turns what it wrote into fast, System-1 recall, which reaches the released HyperMem's score on its own and surpasses every evaluated system when plugged into it.
Simulating richer lives is the next step towards memory that, like Proust's madeleine, recalls what matters without being asked.
\section*{Limitations}

\textbf{Scope of the end-to-end evaluation.} We establish the end-to-end state of the art on LoCoMo-Plus under its official protocol; it is currently the benchmark built for associative recall in long conversations with an official evaluation protocol; evaluating on further benchmarks of this kind as they appear is a natural extension.
\textbf{Coverage of the simulated lives.} The life simulator writes English conversations, and the lives it writes follow the generating LLMs' notion of how a typical life unfolds. Associations that are specific to a culture, a community or an individual may therefore be learned less well, and extending the simulator to other languages and to more diverse lives is left to future work.

\section*{Ethics Statement}

Long-term memory makes an assistant remember personal information about its users, and better associative recall means that a remark can surface a related memory the user did not mention explicitly. Deployments should therefore let users inspect and delete what is remembered. All training data in this work is written by an LLM for simulated personas and contains no real personal data; the benchmarks we evaluate on are public and are used under their licenses. The human-written triggers in Appendix~\ref{app:full} were written by four graduate students who volunteered without payment, following written instructions; they composed fictional chat messages, and no personal information was collected.

\bibliography{refs}

\appendix
\raggedbottom\setcounter{dbltopnumber}{3}\renewcommand{\dbltopfraction}{0.9}\renewcommand{\dblfloatpagefraction}{0.6}

\section{Life Simulator: Prompts, Samples and Decontamination}
\label{app:simulator}

\textbf{Sources.} Table~\ref{tab:app-data} summarizes the three sources. Every call simulates one persona, whose age, occupation, place and the life areas and statement kinds to cover are sampled from fixed lists with a fixed seed. Each persona produces about ten earlier statements; each statement comes with three later messages and one direct question. A training pair maps a later message (or the direct question) to its earlier statement, and the persona's other statements act as natural hard negatives. The three sources together cost 3.13 dollars. %
None of the prompts copies the wording or examples of any benchmark.

\begin{table*}[t]
\centering
\caption{The three life-simulator sources. \emph{Pair cos} is the mean cosine (untrained Qwen3-Embedding-4B) between an earlier statement and its later messages on a 150-pair pilot; lower means less similar. Rows are training pairs (later message $\to$ earlier statement, plus one direct question per statement) after decontamination; they include every later message the simulator returned, and some statements received more than three. v2 and v2d happen to keep the same number of rows.}
\label{tab:app-data}
\setlength\tabcolsep{3pt}
{\footnotesize
\setlength{\aboverulesep}{0pt}\setlength{\belowrulesep}{0pt}
\renewcommand{\arraystretch}{1.1}
\begin{tabular}{llccccc}
\toprule
\headrow Source & Generator & Calls & Cost (\$) & Statements & Pair cos & Rows kept (removed) \\
\midrule
v2 (static facts) & gpt-4.1-mini & 800 & 1.61 & 8{,}000 & 0.422 & 31{,}697 ($\approx$1\%) \\ %
v3 (trajectories) & DeepSeek-V4.1-Flash & 800 & 0.79 & 7{,}975 & 0.398 & 32{,}419 (49) \\ %
v2d (v2 prompt) & DeepSeek-V4.1-Flash & 800 & 0.73 & 8{,}136 & -- & 31{,}697 (1{,}092) \\ %
\bottomrule
\end{tabular}}
\end{table*}

\textbf{Prompt of v3 (trajectories).} Placeholders in braces are filled per call.
\begin{quote}\footnotesize
You are simulating one real person's life over about a year, as it comes up in casual chats with a close friend. Persona: a \texttt{\{age\}}-year-old \texttt{\{job\}} living in \texttt{\{place\}}.
(1)~Write \texttt{\{n\}} different things this person told the friend early on, each as one or two casual first-person sentences. Every statement must be SPECIFIC (name the actual goal, plan, event, situation, person, feeling or principle and why it matters to them). Spread them over these life areas: \texttt{\{domains\}}. Use these kinds (one each, in any order): \texttt{\{kinds\}}.
(2)~For each statement, write 3 things the person says to the friend weeks or months later, in new situations, that are part of how their life actually unfolded from it: progress, a setback, giving up, a change of heart, a consequence, a new choice it shaped, or a dilemma it creates. Make the 3 developments differ from each other. HARD RULES for the later messages: The person does not refer back: they must NOT mention, hint at or allude to the earlier statement, its topic, its key objects, people, places or feelings. No ``remember when'', ``like I said'', ``because of that'', ``given my \dots''. Share no content words with the earlier statement; no paraphrase. A stranger reading the later message must see no link at all; a friend who remembers the earlier statement would immediately see how it explains or bears on the later one, and would answer differently because of it. Each must read as a normal standalone chat message, and for each later message only this one earlier statement of the person should matter. Good example of the spirit (do not reuse): earlier ``I'm putting every spare dollar toward a used sailboat so I can sail the coast next summer.'' $\to$ later ``Sold my car this morning, the bus is honestly fine.''
(3)~For each statement, also write one short question the friend might ask that plainly refers to it.
Vary tone, length and specificity. Avoid clich\'es and avoid repeating the same kind of development. Return JSON: \texttt{\{"facts": [\{"fact": ..., "later": [...], "direct": ...\}, ...]\}}
\end{quote}

\textbf{Prompt of v2 and v2d (static facts).} The v2 prompt has the same three-step structure and the same no-leak rules. It differs in what is simulated: the persona chats with an AI assistant, step~(1) asks for specific personal facts (a named condition, medication, diet, practice, rule, obligation, possession, place, event, goal or limitation) drawn from 22 fact kinds, and step~(2) asks for requests, plans, purchases, invitations, updates or complaints in which ``a knowledgeable assistant who remembers [the fact] should change its answer because of real-world consequences''. Its in-prompt example is a blood thinner followed by a request for mountain-biking tips. v2d reuses this prompt and the v2 persona seeds with a different generator. %

\textbf{Samples.} From v3 (a 51-year-old pilot): the earlier statement ``Signed up for the Tuesday/Thursday masters swim at the rec pool \dots\ I'm weirdly nervous about being the slowest guy in the lane'' is paired with the later message ``Knees are shot from all the yardage, so I moved to the bike for a while.'' From v2 (a 73-year-old pharmacist): ``I have type 2 diabetes and follow a strict low-carb diet'' is paired with ``I need ideas for snacks to bring on my morning train ride that won't cause an afternoon crash.'' %

\textbf{Decontamination.} We embed every training text with the untrained Qwen3-Embedding-4B and delete any training pair whose cosine to any test text of LoCoMo-Plus, LoCoMo-Conv (implicit queries), InMind or ImplexConv~\citep{implexconv} is at least 0.85. %
This removes about 1\% of v2, 49 rows of v3 and 1{,}092 rows of v2d. %
Most v2d removals are generic direct questions such as ``What is my diet?'' that coincide with a test text. %

\textbf{Generator choice.} On a pilot of five personas (150 pairs) per generator, the mean cosine between an earlier statement and its later messages is 0.511 for gpt-4.1-mini with the v3 prompt, against 0.397 for gpt-4.1 and 0.398 for DeepSeek-V4.1-Flash; manual inspection showed that gpt-4.1-mini often continues the earlier topic directly. %
We therefore generate v3 with DeepSeek-V4.1-Flash and disable its thinking mode, which truncates the JSON output. %

\section{Training, Evaluation and Latency Details}
\label{app:hparams}

Table~\ref{tab:app-hparams} lists the hyperparameters. The memory bank of each training step is the set of all training facts, encoded once by the frozen backbone. For every query in a batch, facts from other sources and facts whose cosine to the gold fact exceeds 0.9 are masked out of the softmax. %
After holding out 5\% of the personas, the three sources give 91{,}067 training pairs; one epoch of the 4B model takes 711 steps. %
Queries carry the instruction ``Given a new message in a long conversation, retrieve earlier conversation content that is relevant to it'' in the Qwen3-Embedding format, both in training and at test time; memories carry no instruction. %

\begin{table}[h]
\centering
\caption{Hyperparameters of the query-side association adapter (identical for all backbones and seeds unless stated).}
\label{tab:app-hparams}
\footnotesize
\setlength{\aboverulesep}{0pt}\setlength{\belowrulesep}{0pt}
\renewcommand{\arraystretch}{1.1}
\begin{tabular}{lp{0.45\columnwidth}}
\toprule
\headrow Item & Value \\
\midrule
Backbone & Qwen3-Embedding-4B (also 0.6B, 8B)~\citep{qwen3emb} \\
Adapter & LoRA~\citep{lora}, $r=32$, $\alpha=64$, dropout 0.05 \\ %
LoRA targets & q, k, v, o, gate, up, down projections \\ %
Trainable parameters (4B) & 66.1M \\ %
Loss & InfoNCE on the summed score, $\tau=0.05$ \\ %
Negatives & all same-source training facts \\ %
Near-duplicate mask & cosine to gold fact $>0.9$ \\ %
Optimizer & AdamW, lr $10^{-4}$, weight decay 0.01 \\ %
Schedule & 5\% linear warm-up, cosine decay \\ %
Batch size, epochs & 128, 1 (711 steps for 4B) \\ %
Max sequence length & 128 tokens \\ %
Gradient clipping & 1.0 \\ %
Validation split & 5\% of simulated personas, monitored only; the final checkpoint is used \\ %
Precision & bf16, gradient checkpointing \\ %
Seeds & 0, 1, 2 \\ %
\bottomrule
\end{tabular}
\end{table}

\textbf{Retrieval evaluation.} A LoCoMo-Plus memory base is one LoCoMo conversation plus the inserted cue dialogue, with one memory per turn (about 600 per item); the trigger itself is not stored. The rank of an item is one plus the number of base turns scored above the best cue turn. %
The same units and ranks are used for all local retrievers.

\textbf{Latency.} On one RTX 4090D (bf16, batch size 1, tokenization included, median of 200 runs), the untrained 4B encodes a query in 37.9\,ms. \ourmethod\ with an unmerged LoRA, as used in all experiments, takes 145.0\,ms; with the LoRA merged into a second copy of the weights it takes 70.5\,ms, about two backbone passes. %
Scoring about 600 memories and taking the top 10 costs 0.04\,ms. %

\section{Official Protocol and Reproduced Baselines}
\label{app:protocol}

\textbf{Protocol.} We evaluate all 401 cognitive items of LoCoMo-Plus; item $i$ uses LoCoMo conversation $i \bmod 10$, as in the official code. The reader input is built by the official \texttt{\_build\_model\_input} with the Cognitive instruction; every selected turn is rendered as \texttt{X said, "\dots"} in time order, and the trigger is the last line. A self-test confirms that our formatting is byte-identical to the official stitching on all 401 items. %
For $K$ memories we take the top-$K$ retrieved turns, add the preceding and following turn of the same session to each, and sort the result by time. This gives about 8, 14 and 28 dialogue lines and about 375 and 565 reader input tokens for $K=3$ and $K=5$. %
The reader uses temperature 0.3 and at most 1{,}024 tokens; the judge uses temperature 0, at most 512 tokens, the verbatim official judge input and the official parser, and only \emph{correct} scores a point. %
Photo captions stored with some turns are used for retrieval but removed from the reader input, matching the official whole-conversation input. %
All numbers use the verbatim official judge input, and significance is an exact McNemar test on per-item labels~\citep{mcnemar1947}.

\textbf{Published whole-conversation numbers.} The LoCoMo-Plus paper reports 21.05 for GPT-4o reading the whole conversation and 26.06 for Gemini-2.5-Pro under its protocol. %
Running the released code, GPT-4o with the whole conversation scores 44.8 on a 58-item subset (every seventh item; $p=7\times10^{-5}$ against 21.05), and gpt-4.1-mini scores 44.9 on all 401 items. %
JustMem reports 21.70 for the whole-conversation baseline with gpt-4.1-mini as reader and judge; the released code with the same models gives 37.7 ($p=7\times10^{-7}$). %
We therefore compare only against baselines re-run in the same code, and list published numbers from other protocols separately.

\section{Full Results}
\label{app:full}

\begin{table}[t]
\centering
\caption{LoCoMo-Plus under the official protocol (gpt-4.1-mini reader unless stated) for $K=3,5,10$ (official gemini-2.5-flash judge). FullText has no $K$.}
\label{tab:app-full}
\setlength\tabcolsep{2.5pt}
{\footnotesize
\setlength{\aboverulesep}{0pt}\setlength{\belowrulesep}{0pt}
\renewcommand{\arraystretch}{1.1}
\begin{tabular}{lccc}
\toprule
\headrow Retriever & $K$=3 & $K$=5 & $K$=10 \\
\midrule
\oursrow \ourmethod\ (4B), seed 0 & 52.4 & 53.1 & 54.1 \\ %
\oursrow \ourmethod\ (4B), seed 1 & 51.6 & 53.9 & 51.9 \\ %
\oursrow \ourmethod\ (4B), seed 2 & 53.1 & 54.4 & 53.9 \\ %
\oursrow \ourmethod\ (4B), mean & 52.4 & 53.8 & 53.3 \\ %
Untrained Qwen3-Embedding-4B & 39.7 & 47.4 & 49.9 \\ %
gemini-embedding & 41.6 & 45.9 & 46.4 \\ %
text-embedding-3-large & 41.1 & 46.4 & 47.4 \\ %
Qwen3-Embedding-8B & 37.4 & 43.4 & 42.6 \\ %
bge-m3 & 21.9 & 25.7 & 34.4 \\ %
FullText (whole conversation) & \multicolumn{3}{c}{44.9} \\ %
\midrule
\oursrow \ourmethod\ (8B) & 53.4 & 55.4 & 51.1 \\ %
Untrained 8B (local run) & 36.7 & 41.6 & 41.6 \\ %
\midrule
\multicolumn{4}{l}{\emph{GPT-4o reader}} \\
\oursrow \ourmethod\ (4B), seed 0 & 50.1 & -- & -- \\ %
Untrained Qwen3-Embedding-4B & 40.4 & -- & -- \\ %
gemini-embedding & 35.2 & -- & -- \\ %
\bottomrule
\end{tabular}}
\end{table}

\begin{table}[h]
\centering
\caption{Plug-in results under the official protocol, including T-Mem at its default of 10 scenes. Only the retrieval encoder differs within each system (T-Mem's memories are re-encoded by the frozen encoder); \ourmethod\ is seed 0.}
\label{tab:app-system}
\setlength\tabcolsep{2.5pt}
{\footnotesize
\setlength{\aboverulesep}{0pt}\setlength{\belowrulesep}{0pt}
\renewcommand{\arraystretch}{1.1}
\begin{tabular}{llc}
\toprule
\headrow System & Query encoder & Score \\
\midrule
HyperMem & original & 52.9 \\ %
HyperMem & untrained 4B (with instruction) & 60.6 \\ %
\oursrow HyperMem & \ourmethod & \textbf{66.6} \\ %
\midrule
T-Mem, 3 scenes & original (bge-m3) & 33.7 \\ %
T-Mem, 3 scenes & untrained 4B & 53.9 \\ %
\oursrow T-Mem, 3 scenes & \ourmethod & \textbf{59.9} \\ %
T-Mem, 10 scenes & original (bge-m3) & 35.2 \\ %
T-Mem, 10 scenes & untrained 4B & 45.4 \\ %
\oursrow T-Mem, 10 scenes & \ourmethod & 49.9 \\ %
\bottomrule
\end{tabular}}
\end{table}

Table~\ref{tab:app-full} gives all retrievers for $K=3,5,10$, and Table~\ref{tab:app-system} the plug-in results, including T-Mem at its default of 10 scenes, where \ourmethod\ is 14.7 points above the original ($p=8\times10^{-8}$) and 4.5 points above the untrained 4B ($p=0.11$). %
Plugged into HyperMem, the other two seeds of \ourmethod\ score 68.6 and 63.3 (66.2 on average over three seeds), each above every other system we evaluate and significantly above the released HyperMem ($p\le9\times10^{-5}$); against the instructed HyperMem, the differences are $+8.0$ ($p=0.0003$) and $+2.7$ ($p=0.25$), and the three-seed mean is $+5.6$ (95\% CI $[+2.0,+9.2]$, paired sign-flip test on per-question accuracy averaged over seeds, $p=0.0035$). %

\textbf{Human-written triggers.} To test whether the gain comes from the style of LLM-written questions, four annotators rewrote the trigger of 200 LoCoMo-Plus items (50 per relation type), seeing only the cue dialogue, the time gap and the relation type but not the original trigger, and a fifth person checked each one; memories, reader and judge are unchanged. %
Human triggers share far more words with their cues (BM25 finds 39\% of cues in the top 10, against 6\% for the original triggers), so every retriever improves. %
\ourmethod\ improves as well rather than degrading: it ranks the cue first for 80.0\% of the items, up from 59.5\% ($p<10^{-4}$), the highest among all retrievers and significantly above its untrained backbone ($+7.5$, $p=0.0013$). %
On the 99 items whose trigger shares no content word with the cue, where similarity fails, it ranks the cue first significantly more often than every other retriever ($p\le0.043$), and its end-to-end gain over the untrained backbone is even larger than on the original triggers ($+12.1$, $p=0.016$, against $+10.4$). %

\textbf{Ordinary LoCoMo QA.} We follow the T-Mem LoCoMo protocol (categories 1--4, 1{,}540 questions, gpt-4o-mini reader and judge, three judgments per question, $K=10$) and compare each retriever with its own untrained backbone by a paired bootstrap and a sign-flip test. %
On 4B, the untrained backbone scores 59.1 and the three seeds of \ourmethod\ score 58.9, 59.2 and 60.0 (differences $-0.2$, $+0.1$ and $+0.9$; $p=0.86$, $0.96$, $0.29$). %
On 8B, the untrained backbone scores 63.5 and \ourmethod\ 60.9, a drop of 2.6 points (95\% CI $-4.1$ to $-1.0$, $p=0.0015$); the share of questions with evidence in the top 10 drops from 72.5\% to 67.4\%. %
The loss comes from direct single-hop questions, where the association term pulls the ranking away from the literally best-matching turn towards related situations; for the earlier v2-only 4B retriever, 91 questions lost their evidence from the top 10 and 62 gained it. %
The extra v2d source, which removes this cost on 4B, does not remove it on 8B; we conjecture that the stronger literal matching of the 8B backbone makes the cost larger.

\section{T-Mem Under Its Own Protocol}
\label{app:tmem}

We also ran T-Mem with its released code and its own LoCoMo-Plus protocol: one memory base per item built by gpt-4.1-mini, three-way embedding retrieval fused by RRF, the top 10 scenes, a GPT-4o reader, and T-Mem's own answer and judge prompts with gemini-2.5-flash as judge. %
The three arms share the same 401 memory bases and differ only in the embedding used for retrieval; for \ourmethod\ and the untrained 4B, memories are encoded by the frozen 4B. Table~\ref{tab:app-tmem} gives the results. %
With the original bge-m3 retriever, T-Mem scores 51.9; with the untrained 4B, 67.6; with \ourmethod, 70.6 (against bge-m3: $p=1\times10^{-13}$; against the untrained 4B: $p=0.23$). %
When only the top three scenes are given to the reader, \ourmethod\ scores 72.6 against 64.1 for the untrained 4B ($p=1.5\times10^{-4}$) and 41.4 for bge-m3. %
The T-Mem paper reports 74.81 on LoCoMo-Plus. Our run of the released code with the models named in the paper gives 51.9. The released code retrieves 10 scenes with an RRF constant of 30, whereas the paper describes 5 scenes with a constant of 60. %

The released trigger-generation prompt, which is applied to all 401 memory bases, contains two few-shot examples whose cue text is that of LoCoMo-Plus test items \#81 and \#188, and its negative examples include text from items \#3, \#188 and \#306. %
Excluding these items together with \#271 and \#275, whose structure closely matches one of the examples, gives 51.6 for bge-m3 and 70.6 for \ourmethod. %

\begin{table*}[t]
\centering
\caption{T-Mem under its own released code and evaluation protocol (401 LoCoMo-Plus items, GPT-4o reader, T-Mem's answer and judge prompts). Only the embedding used by its retrieval stage changes; the memory bases are shared by all arms. \emph{Cue@k}: the scene containing the cue is among the $k$ retrieved scenes.}
\label{tab:app-tmem}
\setlength\tabcolsep{3pt}
{\footnotesize
\setlength{\aboverulesep}{0pt}\setlength{\belowrulesep}{0pt}
\renewcommand{\arraystretch}{1.1}
\begin{tabular}{lccccccc}
\toprule
\headrow Retriever in T-Mem & All & Causal & Goal & State & Value & Cue@10 & 3 scenes (Cue@3) \\
\midrule
bge-m3 (original) & 51.9 & 64.4 & 44.0 & 59.0 & 40.0 & 70.1 & 41.4 (40.4) \\ %
Untrained Qwen3-Embedding-4B & 67.6 & \textbf{78.2} & 62.0 & 67.0 & 63.0 & 89.3 & 64.1 (72.6) \\ %
\oursrow \ourmethod & \textbf{70.6} & \textbf{78.2} & \textbf{67.0} & \textbf{72.0} & \textbf{65.0} & \textbf{94.5} & \textbf{72.6} (\textbf{84.3}) \\ %
\bottomrule
\end{tabular}}
\end{table*}

\section{Other Benchmarks}
\label{app:other}

\begin{table*}[t]
\centering
\caption{Other benchmarks. Retrieval rows are Recall@10 (\%); end-to-end rows use each benchmark's own protocol. \ourmethod\ is seed 0 in the first two LoCoMo-Conv rows and the three-seed mean in the other rows.}
\label{tab:app-other}
\setlength\tabcolsep{3pt}
{\footnotesize
\setlength{\aboverulesep}{0pt}\setlength{\belowrulesep}{0pt}
\renewcommand{\arraystretch}{1.1}
\begin{tabular}{llc>{\columncolor{oursbg}}cc}
\toprule
\headrow Benchmark & Metric & Untrained 4B & \ourmethod & text-emb-3-large \\
\midrule
LoCoMo-Conv (1,540) & Recall@10 (paper metric) & 46.2 & 48.1 & \textbf{49.6} \\ %
LoCoMo-Conv & answer score, $K$=10 & 0.440 & \textbf{0.456} & 0.453 \\ %
LoCoMo-Conv, user turns only & answer score, $K$=10 & 0.456 & 0.482 & \textbf{0.489} \\ %
\midrule
InMind single fact (125) & indirect query & 5.6 & \textbf{27.7} & 0.8 \\ %
InMind single fact & direct query & \textbf{96.8} & 93.1 & 80.8 \\ %
\bottomrule
\end{tabular}}
\end{table*}

Table~\ref{tab:app-other} collects the remaining benchmarks; each uses its own protocol and memory units.

\textbf{LoCoMo-Conv}~\citep{locomoconv}. We follow the released pipeline (one memory per turn, top 10, gemma-4-31B-it reader, the three-level gpt-5.4-mini judge). Our reproduction of its MiniLM baseline gives Recall@10 32.1 and an answer score of 0.332, against 31.2 and 0.335 in the paper. %
\ourmethod\ is 1.6 points above the untrained 4B ($p=0.11$) and ties text-embedding-3-large ($p=0.86$). %
Restricting candidates to the user's own turns, which the official prompt identifies, raises every retriever by 4.6 to 6.6 Recall@10 points. With this filter all three seeds are significantly above the untrained 4B ($p=0.011$, $0.016$, $0.0021$) and tie text-embedding-3-large ($p\geq0.42$). %
90.4\% of the evidence turns are spoken by the user, and the typical retrieval error is the right topic but the wrong turn. %

\textbf{InMind}~\citep{inmind}. In the single-fact setting, one personal fact is hidden among about 2{,}600 LongMemEval~\citep{longmemeval} user messages and retrieved by an indirect or a direct question (125 items). \ourmethod\ raises indirect Recall@10 from 5.6\% to 27.7\% (seeds 28.8, 28.8, 25.6) while direct Recall@10 moves from 96.8\% to 93.1\%. %

\section{Ablations}
\label{app:ablation}

Table~\ref{tab:ablation} gives the full ablation of Section~\ref{sec:rq4}, and Table~\ref{tab:app-relation} gives cue Recall@10 per relation type.

Placing the same supervision in a reranker on top of the untrained 4B reaches a cue Recall@10 of only 75.3, below the query-side retriever alone (78.8), while costing 20 extra 4B forward passes per query. %
On ordinary LoCoMo QA with the 4B backbone, the v2-only retriever loses 2.7 points against the untrained backbone ($p=0.0025$) and the two-source recipe 1.6 points on average, whereas the three-source recipe changes it by $+0.3$ on average. %
At $K=5$, two of the three seeds are significantly above the untrained 4B end to end; at $K=10$, none is. %

\section{Routes That Do Not Learn Association}
\label{app:routes}

Table~\ref{tab:routes} lists every route of Section~\ref{sec:rq3} with its significance against \ourmethod{}.
The EAR-style adapter follows the structure of EAR~\citep{ear}: on the frozen 4B, two $2560{\times}2560$ matrices initialized at zero map the query to $(I+W_q)q$ and the memory to $(I+W_m)m$, scored by cosine (13.1M parameters). It is trained on our data with our loss, negatives and temperature, and its epoch is chosen on held-out simulated personas, not on any benchmark. %
The reasoning-oriented retrievers reach a cue Recall@10 between 33.9 and 60.6, against 71.1 for the untrained backbone and 82.8 for \ourmethod{} (seed~0). %
The best QueryLink configuration passes 1,301 answer-input tokens per query and is 16.5 points below \ourmethod{}. %
The two rerankers cut the share of cues ranked first by \ourmethod{} (seed~0) from 62.3\% to 35.2\% (Qwen3-Reranker-4B) and 17.7\% (MemReranker-4B). %

\section{Reproduction of QueryLink and HyperMem}
\label{app:repro}

\textbf{QueryLink}~\citep{querylink}. We run the release branch of the official repository (commit \texttt{cbf1e0d}) with its three prompts verbatim, gpt-4o-mini as the rewriting model and all-MiniLM-L6-v2 as the encoder. %
Each turn is a unit, and its views are generated from the turn and the two neighbouring turns of the same session, as in its LoCoMo adapter; the query is the trigger. %
The upstream code parses the rewriting output with a strict JSON parser, which drops the LLM views whenever the model wraps its answer in a code block; we keep this behaviour, so about 55.6\% of the turns retain only the raw-text view. Stripping the code-block markers before parsing, which keeps all views, does not help QueryLink: its three configurations then score 20.9, 33.7 and 30.4 instead of 23.7, 35.9 and 34.4. %
The three rows are the original configuration restricted to three memories, its native multi-view union (about 1{,}600 context tokens), and its rewrites encoded by the frozen Qwen3-Embedding-4B. A fourth row uses only the ``guess related events'' output as a HyDE-style query~\citep{hyde} on the same frozen 4B and scores 41.9. %
Selecting three memories from several views uses RRF ($k=60$) over per-view maxima. %
Writing costs about 590 LLM calls and 243K tokens per memory base; reading costs two calls per query. %

\textbf{HyperMem}~\citep{hypermem}. We run the official repository (commit \texttt{15c7009}) with gpt-4.1-mini at temperature 0 for extraction and the Qwen3-Embedding-4B served by OpenRouter, whose vectors match our local model at cosine about 0.9998. %
The ten base conversations are built once with the official functions; for each item the cue dialogue is segmented, assigned to topics and turned into facts by the official extractors and merged into the base hypergraph. %
Retrieval uses the code defaults (candidate pool 100; 15 topics, 20 episodes and 30 facts; reranking off), and the reader receives episodes and facts as in its answer stage. %
Building one memory base costs 1{,}095 LLM calls and about 3.55M tokens on average. %
For the plug-in, the query is embedded once by the new encoder and that vector replaces all three query encodings of the retrieval stage; the BM25 half of its hybrid search still uses the raw query text. %
The untrained 4B plug-in adds the Qwen3-Embedding query instruction, which the original configuration omits; this accounts for 7.7 of the 13.7-point gain, and learned association for the other 6.0. %

\begin{table*}[tp]
\centering
\caption{Cue Recall@10 (\%) on LoCoMo-Plus by relation type. v2 is a single simulator source (gpt-4.1-mini, concrete facts followed by requests to an assistant); the full recipe mixes v2, v3 and v2d (seed~0). $^{*}$Significantly above the untrained backbone of the same size ($p<0.05$, exact McNemar).}
\label{tab:app-relation}
\footnotesize
\setlength{\tabcolsep}{4pt}
\setlength{\aboverulesep}{0pt}\setlength{\belowrulesep}{0pt}
\renewcommand{\arraystretch}{1.1}
\begin{tabularx}{\textwidth}{>{\hsize=1.5\hsize}X*{5}{>{\centering\arraybackslash\hsize=0.9\hsize}X}}
\toprule
\headrow Retriever & Causal & State & Goal & Value & All \\
\midrule
Untrained 0.6B & 50.5 & 41.0 & 33.0 & 19.0 & 35.9 \\
0.6B, v2 only & 72.3$^{*}$ & 61.0$^{*}$ & 53.0$^{*}$ & 51.0$^{*}$ & 59.4$^{*}$ \\
\midrule
Untrained 4B & 89.1 & 65.0 & 67.0 & 63.0 & 71.1 \\
4B, v2 only & 90.1 & 82.0$^{*}$ & 72.0 & 71.0$^{*}$ & 78.8$^{*}$ \\
\oursrow 4B, full recipe & 90.1 & 85.0$^{*}$ & 79.0$^{*}$ & 77.0$^{*}$ & 82.8$^{*}$ \\
\bottomrule
\end{tabularx}
\end{table*}
\begin{table*}[tp]
\centering
\caption{Ablations of the design choices in Section~\ref{sec:method}. Columns 2--3 are retrieval Recall@10 (\%); \emph{LP} (LoCoMo-Plus) tests associative recall and \emph{LoCoMo} ordinary lookup. The last column is the end-to-end change on ordinary LoCoMo QA relative to the untrained backbone (mean over seeds where three exist). Shaded rows are the configuration used in the paper for that block; bold marks the best LP Recall@10 in each block.}
\label{tab:ablation}
\setlength\tabcolsep{3pt}
{\footnotesize
\setlength{\aboverulesep}{0pt}\setlength{\belowrulesep}{0pt}
\renewcommand{\arraystretch}{1.1}
\begin{tabularx}{\textwidth}{>{\hsize=2.2\hsize}X*{3}{>{\centering\arraybackslash\hsize=0.6\hsize}X}}
\toprule
\headrow Variant & LP & LoCoMo & LoCoMo QA $\Delta$ \\
\midrule
\multicolumn{4}{l}{\emph{(a) Which side to train, and which score to train on} \quad [0.6B, v2 data]} \\
Untrained 0.6B & 35.9 & 67.3 & -- \\ %
Both sides, full fine-tuning & 49.6 & 55.1 & -- \\ %
Both sides, LoRA & 50.1 & 56.1 & -- \\ %
Query side, association term trained alone & 54.1 & 64.5 & -- \\ %
\oursrow Query side, residual (summed) score & \textbf{59.4} & 60.8 & -- \\ %
\midrule
\multicolumn{4}{l}{\emph{(b) Mixing simulator sources} \quad [0.6B, query side, residual score]} \\
v2 only & 59.4 & 60.8 & -- \\ %
v3 only & 57.6 & 67.3 & -- \\ %
v2 + v3, all negatives & 50.1 & 65.3 & -- \\ %
v2 + v3, same-source negatives & \textbf{63.6} & 64.0 & -- \\ %
\oursrow v2 + v3 + v2d, same-source negatives & 60.6 & 66.2 & -- \\ %
\midrule
\multicolumn{4}{l}{\emph{(c) Main backbone} \quad [4B]} \\
Untrained 4B & 71.1 & 66.1 & 0.0 \\ %
Both sides, LoRA (v2) & 69.6 & 68.9 & -- \\ %
Query side, v2 only & 78.8 & 62.5 & $-$2.7$^{*}$ \\ %
Query side, v2 + v3, same-source (3 seeds) & \textbf{84.3} & 61.6 & $-$1.6 \\ %
\oursrow Query side, v2 + v3 + v2d, same-source (3 seeds) & 83.1 & 64.2 & $+$0.3 \\ %
\bottomrule
\end{tabularx}}
\vspace{2pt}
\parbox{\textwidth}{\footnotesize $^{*}$\,$p<0.01$, paired sign-flip test against the untrained 4B. For the v2\,+\,v3 recipe, two of three seeds are significantly below the untrained 4B ($p=0.03$ and $0.025$); for the full recipe no seed is ($p\geq0.29$).}
\end{table*}

\begin{table*}[tp]
\centering
\caption{Routes that do not learn association, on LoCoMo-Plus (official protocol, $K{=}3$). \textbf{Read}: LLM calls per query at retrieval time. \textcolor{dropc}{$\downarrow$}: score drop relative to \ourmethod{} (seed~0), computed before rounding; $p$: exact McNemar test against it. Rerankers re-score the top 50 candidates of \ourmethod{}.}
\label{tab:routes}
\footnotesize
\setlength{\tabcolsep}{3.5pt}
{\footnotesize
\setlength{\aboverulesep}{0pt}\setlength{\belowrulesep}{0pt}
\renewcommand{\arraystretch}{1.1}
\begin{tabularx}{\textwidth}{>{\hsize=1.8\hsize}X*{4}{>{\centering\arraybackslash\hsize=0.8\hsize}X}}
\toprule
\headrow \textbf{Method} & \textbf{Read} & \textbf{Score} & \textbf{$\Delta$} & \textbf{$p$} \\
\midrule
\oursrow \ourmethod{} (seed 0)                & 0 & \textbf{52.4} & -- & -- \\ %
Untrained Qwen3-Embedding-4B         & 0 & 39.7 & \drop{12.7} & $1.2{\times}10^{-5}$ \\ %
\midrule
\multicolumn{5}{l}{\textit{(i) Read-time LLM rewriting or guessing}} \\
QueryLink, original config. & 2 & 23.7 & \drop{28.7} & $1.6{\times}10^{-19}$ \\ %
QueryLink, native union              & 2 & 35.9 & \drop{16.5} & $1.3{\times}10^{-7}$ \\ %
QueryLink, rewriting + Qwen3-4B      & 2 & 34.4 & \drop{18.0} & $9.0{\times}10^{-9}$ \\ %
HyDE, Qwen3-4B           & 1 & 41.9 & \drop{10.5} & $2.9{\times}10^{-4}$ \\ %
\midrule
\multicolumn{5}{l}{\textit{(ii) Reasoning-oriented retrievers}} \\
ReasonIR-8B          & 0 & 40.6 & \drop{11.7} & $8.3{\times}10^{-5}$ \\ %
ReasonEmbed-Qwen3-4B & 0 & 29.9 & \drop{22.4} & $4.6{\times}10^{-13}$ \\ %
Diver-Retriever-4B      & 0 & 22.4 & \drop{29.9} & $3.5{\times}10^{-21}$ \\ %
Diver-Retriever-4B-1020              & 0 & 31.7 & \drop{20.7} & $1.4{\times}10^{-11}$ \\ %
\midrule
\multicolumn{5}{l}{\textit{(iii) Rerankers on top of \ourmethod{}}} \\
+ Qwen3-Reranker-4B & 0 & 40.9 & \drop{11.5} & $1.6{\times}10^{-4}$ \\ %
+ MemReranker-4B  & 0 & 31.9 & \drop{20.4} & $1.4{\times}10^{-10}$ \\ %
\midrule
\multicolumn{5}{l}{\textit{(iv) Same data, different trained structure}} \\
EAR-style adapter         & 0 & 39.4 & \drop{13.0} & $1.3{\times}10^{-5}$ \\ %
EAR-style adapter, fused score       & 0 & 45.1 & \drop{7.2} & 0.019 \\ %
\bottomrule
\end{tabularx}}
\end{table*}

\end{document}